\documentclass[us]{article} 
\PassOptionsToPackage{numbers, compress}{natbib}

\usepackage{url}
\usepackage{amsthm}
\usepackage{amssymb,dsfont,amsmath,amsthm,mathtools,natbib}
\usepackage{enumitem}
\usepackage{bbm}

\usepackage{algorithm}
\usepackage{algorithmic}
\usepackage{dcolumn}
\usepackage{tikz}
\usetikzlibrary{positioning}
\usepackage[accepted]{icml2021}
\usepackage{booktabs}

\newtheorem{definition}{Definition}
\newtheorem{proposition}{Proposition}
\newtheorem{lemma}{Lemma}
\newtheorem{property}{Property}

\newtheorem{remark}{Remark}

\usepackage{xspace}
\newcolumntype{d}[1]{D{.}{.}{#1}}

\def\R{\mathbb{R}}
\def\1{\mathbbm{1}}

\def\x{\mathbf{x}}
\def\X{\mathbf{X}}
\def\u{\mathbf{u}}
\def\v{\mathbf{v}}
\def\y{\mathbf{y}}

\def\E{\mathbf{E}}

\def\z{\mathbf{z}}

\def\m{\mathbf{m}}
\def\b{\mathbf{b}}
\def\e{\mathbf{e}}

\def\s{\mathbf{s}}

\def\P{\mathbf{P}}
\def\U{\mathbf{U}}
\def\V{\mathbf{V}}

\def\w{\mathbf{w}}
\def\z{\mathbf{z}}

\newcommand{\Ex}{{\rm I\kern-.3em E}}

\newcommand{\Xspace}{\mathcal{X}}
\newcommand{\Dspace}{\mathcal{D}}
\newcommand{\Renyi}{\mathbb{D}}

\newcommand{\algo}{{\mathcal{A}}}
\newcommand{\image}{\text{Im }}

\DeclareMathOperator{\union}{\cup}
\newcommand{\proba}{\mathbb{P}}
\DeclareMathOperator{\expectation}{\mathbb{E}}
\DeclareMathOperator{\variance}{\mathbb{V}}

\newcommand{\argmin}{\mathop{\mathrm{arg\,min}}}

\newcommand{\normal}{{\cal N}}

\newcommand{\rev}[1]{\textcolor{black}{#1}}

\usepackage[colorinlistoftodos,bordercolor=orange,backgroundcolor=orange!20,linecolor=orange,textsize=scriptsize]{todonotes}

\newcommand{\livain}[1]{\todo[inline]{{\textbf{L:} \emph{#1}}}}

\icmltitlerunning{Differentially Private Sliced Wasserstein Distance}
\begin{document}

\twocolumn[
\icmltitle{Differentially Private Sliced Wasserstein Distance}

\begin{icmlauthorlist}
	\icmlauthor{Alain Rakotomamonjy}{criteo,univ}
		\icmlauthor{Liva Ralaivola}{criteo}
\end{icmlauthorlist}
\icmlaffiliation{criteo}{Criteo AI Lab, Paris, France}
\icmlaffiliation{univ}{LITIS EA4108, Université de Rouen Normandie, Saint-Etienne du Rouvray, France}
\icmlcorrespondingauthor{Alain Rakotomamonjy}{alain.rakoto@insa-rouen.fr}
\vskip 0.3in
]
\printAffiliationsAndNotice{} \begin{abstract}

Developing machine learning methods that are privacy preserving is today a central topic of research,
with huge practical impacts. Among the numerous ways to address privacy-preserving learning, we here take the perspective of computing the divergences between distributions under the Differential Privacy (DP) framework --- being able
to compute divergences between distributions is pivotal for many machine learning problems, such as
 learning generative models or domain adaptation problems. 
 Instead of resorting to the popular gradient-based sanitization method for DP, we tackle the problem at its roots by focusing on the Sliced Wasserstein Distance and seamlessly making it differentially private. Our main contribution is as follows: we analyze the property of adding a Gaussian perturbation 
to the intrinsic randomized mechanism of the Sliced Wasserstein Distance, and we establish the sensitivity
of the resulting differentially private mechanism. One of our important findings is that this DP mechanism transforms the Sliced Wasserstein distance into another distance, that we call the Smoothed Sliced Wasserstein Distance. This new differentially private distribution distance can be plugged into generative models and
domain adaptation algorithms in a transparent way, and we empirically show that
it yields highly competitive performance compared with gradient-based DP approaches from the literature,
with almost no loss in accuracy for the domain adaptation problems that we consider.
\end{abstract}

\section{Introduction}
Healthcare and computational advertising are examples of domains
that could find a tremendous benefit from the continous advances made in Machine Learning (ML).
However, as ethical and regulatory concerns become prominent in these areas, 
there is the need to devise privacy preserving mechanisms allowing i) to
 prevent the access to individual and critical data and ii) to still
leave the door open to the use of elaborate ML methods.
Differential privacy (DP) offers a sound privacy-preserving framework to
tackle both issues and effective DP mechanisms have been designed for, e.g., 
logistic regression and Support Vector Machines 
\cite{rubinstein2009learning,chaudhuri2011differentially}.

Here, we address the problem of devising a differentially private distribution
distance with, in the hindsight, tasks such as learning generative models
and domain adaptation ---which both may rely on a relevant distribution
distance \cite{lee2019sliced,deshpande2018generative}. In particular, we
propose and analyze a mechanism that transforms the sliced Wasserstein distance  (SWD)
\cite{rabin2011wasserstein} into a differentially private distance while 
retaining the scalability advantages and metric properties of the base SWD.
The key ingredient to our contribution: to take advantage of the combination of the embedded sampling 
process of SWD and the so-called Gaussian mechanism.

Our contributions are as follows: i) we analyze the effect of a Gaussian mechanism on the sliced Wasserstein distance and 
we establish the DP-compliance of the resulting mechanism DP-SWD; ii) 
we show that DP-SWD boils down to what we call {\em Gaussian smoothed
SWD}, that inherits some of the key properties of a distance,  a novel result that has value on its own;  iii) extensive empirical analysis on domain adaptation and generative modeling tasks show that the proposed
DP-SWD is competitive, as
 we achieve DP guarantees without almost no loss in accuracy in  
domain adaptation, while being the first to present
a DP generative model on the $64\times 64$ RGB  CelebA dataset.
\paragraph{Outline.} 
Section~\ref{sec:back} states the problem we are interested in and provides
 background on differential privacy and the sliced Wasserstein distance. 
In Section \ref{sec:methods}, we analyze the DP guarantee of
 {\em random direction projections} and we characterize the resulting Gaussian Smoothed Sliced 
Wasserstein distance. Section \ref{sec:models} discusses how this distance can be plugged into domain 
adaptation and generative model algorithms. After discussing
related works in Section \ref{sec:related}, Section \ref{sec:expe} presents empirical
results, showing our ability to effectively learn under DP constraints.

\section{Problem Statement and Background}
\label{sec:back}

\subsection{Privacy, Gaussian Mechanism and Random Direction Projections}
We start by stating the main problem we are interested
in: to show the privacy properties of the 
random mechanism
$${\cal M}(\X) = \X\U + \V,$$
where $\X\in\R^{n\times d}$ is a matrix (a dataset), $\U\in\R^{d\times k}$ a random
matrix made of $k$ uniformly distributed unit-norm vectors of $\R^d$
and $\V\in\R^{n\times k}$ a matrix of $k$ zero-mean Gaussian vectors (also called
the {\em Gaussian Mechanism}).

We show that ${\cal M}$ is differentially private and that
it is the core component of the Sliced Wassertein Distance (SWD) computed thanks
to {\em random projection directions} (the unit-norm matrix $\U$) and, 
in turn, SWD inherits\footnote{This is a slight abuse of vocabulary
as the Slice\rev{d} Wasserstein Distance takes two inputs and not only one.} 
the differential private property of ${\cal M}$. In the way, we show that
the population version of the resulting differentially private SWD
{\em is} a distance, that we dub the Gaussian Smoothed SWD. 

\subsection{Differential Privacy (DP)}
DP is a theoretical framework to
analyze the privacy guarantees of algorithms. It rests
on the following definitions.

\begin{definition}[Neighboring datasets] Let $\Xspace$ (e.g. $\Xspace = \R^d$)  be a 
	{\em domain} and $\Dspace\doteq\union_{n=1}^{+\infty}\Xspace^n$.
	$D, D'\in \Dspace$ are {\em neighboring datasets} if $|D|=|D'|$ and they
	differ from one record.
\end{definition}

\begin{definition}[\citet{dwork2008differential}]
	\label{def:dp} Let $\varepsilon,\delta > 0$. 
	Let $\mathcal{A}:\Dspace\to \image \algo$ be a {\em randomized} algorithm, where $\image \algo$ is the image
	of $\Dspace$ through $\algo$. $\algo$ is $(\varepsilon,\delta)$-differentially private, or $(\varepsilon,\delta)$-DP, if for all neighboring datasets $D,D^\prime\in\Dspace$ and for all 
	sets of outputs	$\mathcal{O}\in\image \algo$, the following inequality holds:
	$$
	\proba[\mathcal{A}(D) \in \mathcal{O}] \leq 	e^\varepsilon \proba[\mathcal{A}(D^\prime) \in \mathcal{O}] + \delta
	$$
	where the probability relates to the randomness of $\algo$.
\end{definition}

\begin{remark}\label{rem:distribution} Note that given $D\in\Dspace$ and a randomized 
	algorithm $\algo:\Dspace\to \image \algo$, $\algo(D)$
	defines a distribution $\pi_D:\image \algo\to[0,1]$ on (a subspace of) $\image \algo$ with $$\forall 
 \mathcal{O}\in\image \algo,\, \pi_D(\mathcal{O})\propto\proba[\mathcal{A}(D) \in \mathcal{O}],$$
 where $\propto$ means equality up to a normalizing factor.
\end{remark}

The following notion of privacy, proposed by \citet{mironov2017}, which is based on 
Rényi $\alpha$-divergences and its connections to $(\varepsilon,\delta)$-differential privacy
 will ease the exposition of our results (see also \cite{asoodeh2020three,balle18a,wang2019subsampled}):
\begin{definition}[\citet{mironov2017}]\label{def:rdp}
Let $\varepsilon> 0$ and $\alpha>1$. A randomized algorithm $\mathcal{A}$ is  $(\alpha, \varepsilon)$-Rényi differential private or  $(\alpha, \varepsilon)$-RDP, if for any neighboring datasets $D,D^\prime\in\Dspace$,
$$
\Renyi_\alpha\left(\mathcal{A}(D)\| \mathcal{A}(D^\prime)\right) \leq \varepsilon
$$
where $\Renyi_\alpha(\cdot\|\cdot)$ is the Rényi $\alpha$-divergence \cite{renyi1961} between two distributions (cf. Remark~\ref{rem:distribution}).
\end{definition}

\begin{proposition}[\citet{mironov2017}, Prop. 3] An $(\alpha, \varepsilon)$-RDP 
	mechanism is also $(\varepsilon + \frac{log(1/\delta)}{\alpha-1},\delta)$-DP, $\forall\delta\in(0,1)$. 
\end{proposition}

\begin{remark}
\label{rem:gaussian_mechanism}
	A folk method to make up an (R)DP algorithm based a function $f:\Xspace\to\R^d$ 
is the {\em Gaussian mechanism} $\mathcal{M}_\sigma$ defined as follows:
$$\mathcal{M}_\sigma f(\cdot) = f(\cdot) + \v$$
where $\v\sim\mathcal{N}(0,\sigma^2I_d)$. If $f$ has $\Delta_2$- (or $\ell_2$-) {\em sensitivity}
 $$\Delta_2f \doteq \max_{D,D^\prime\text{neighbors}} \|f(D) - f(D^\prime)\|_2,$$
 then $\mathcal{M}_\sigma$ is $\left(\alpha, \frac{\alpha \Delta_2^2f }{2 \sigma^2}\right)$-RDP.
\end{remark}

As we shall see, the role of $f$ will be played by the Random Direction Projections operation or the
Sliced Wasserstein Distance (SWD), a randomized algorithm itself, and the
mechanism to be studied is the composition of
two random algorithms, SWD and the Gaussian mechanism.
Proving the (R)DP nature of this mechanism will rely on a 
high probability bound on the sensitivy of the Random Direction Projections/SWD combined with the result
of Remark~\ref{rem:gaussian_mechanism}.

\subsection{Sliced Wasserstein Distance}

Let $\Omega \in \R^d $ be a probability space and $\mathcal{P}(\Omega)$ the set of
all probability measures over $\Omega$. The Wasserstein distance between two measures
$\mu$, $\nu\in \mathcal{P}(\Omega)$ is based on the so-called Kantorovitch
relaxation of the optimal transport problem, which consists in finding a joint probability
distribution $\gamma^\star \in \mathcal{P}(\Omega \times \Omega)$ such that
\begin{equation}
	\gamma^\star \doteq \argmin_{\gamma \in \Pi(\mu,\nu)} \int_{\Omega \times \Omega} 
	c(x,x^\prime) d\gamma(x,x^\prime)
\end{equation} 
where $c(\cdot,\cdot)$ is a metric on $\Omega$, known as the {\em ground cost} (\rev{which in our case will be the Euclidean distance}), $\Pi(\mu,\nu)\doteq \{ \gamma \in \mathcal{P}(\Omega \times \Omega) |
\pi_{1\#} \gamma=\mu,\pi_{2\#} \gamma=\nu\}$ and $\pi_1, \pi_2$ are
the marginal projectors of $\gamma$ on each of its coordinates. 
The minimizer of this problem is the {\em optimal transport plan} and for $q\geq 1$, the $q$-Wasserstein distance is
\begin{equation}
	W_q(\mu,\nu) = \Big ( \inf_{\gamma \in \Pi(\mu,\nu)} \int_{\Omega \times \Omega} 
	c(x,x^\prime)^q d\gamma(x,x^\prime) \Big)^\frac{1}{q}
\end{equation} 
A  case of prominent interest for our
work is that of one-dimensional measures, for which it was shown by \citet{rabin2011wasserstein,bonneel2015sliced}
that the Wasserstein distance admits a closed-form solution  
\rev{which is $$
W_q(\mu,\nu) \doteq \Big(\int_0^1 |F^{-1}_\mu(z) - F^{-1}_\nu(z)|^qdz
\Big)^\frac{1}{q}
$$
where $F^{-1}_\cdot$ is the inverse cumulative distribution function of the related distribution.
}
This 
combines well with the idea of projecting high-dimensional probability distributions onto random 1-dimensional spaces and
then computing the Wasserstein distance,  an operation which
 can be theoretically formalized through the use of the Radon transform  \cite{bonneel2015sliced}, leading to the so-called Sliced Wasserstein Distance
$$
\text{SWD}_q^q(\mu,\nu) \doteq \int_{\mathbb{S}^{d-1}}  W_q^q(\mathcal{R}_\u \mu,\mathcal{R}_\u \nu) u_d(\u)d\u 
$$
where $\mathcal{R}_\u$ is the Radon transform of a probability distribution so that
\rev{
\begin{equation}
	 \mathcal{R}_\u \mu(\cdot) = \int \mu(\s) \delta(\cdot - \s^\top \u)d\s
	 \label{eq:radon}
\end{equation}}
with $\u \in \mathbb{S}^{d-1}\doteq \{\u \in \R^d : \|\u\|_2=1\}$ be the $d$-dimensional hypersphere and $u_d$ the uniform distribution on  $\mathbb{S}^{d-1}$.

\rev{In practice, we only have access to $\mu$ and $\nu$
 through samples, and the proxy distributions of $\mu$ and $\nu$
 to handle are   $\hat \mu \doteq \frac{1}{n} \sum_{i=1}^{n} \delta_{\x_i}$
and $\hat \nu \doteq \frac{1}{m} \sum_{i=1}^{m} \delta_{\x^\prime_i}.$
By plugging those distributions into Equation~\ref{eq:radon},
it is easy to show that the Radon transform depends only the projection of $\x$ on $\u$. 
Hence, computing the sliced Wasserstein distance amounts to computing the average of 1D Wasserstein distances over a set of random directions $\{\u_j\}_{j=1}^k$, with each 1D probability distribution obtained by projecting a sample (of $\hat \mu$ or $\hat \nu$) on $\u_i$ by $\x^\top \u_i$.}
This gives the following  empirical approximation of SWD
\begin{equation}\label{eq:approxswd}
{\text{SWD}}_q^q \approx \frac{1}{k} \sum_{j=1}^k W_q^q\left(\frac{1}{n} \sum_{i=1}^{n} \delta_{{\x_i}^\top \u_j},
\frac{1}{m} \sum_{i=1}^{m} \delta_{{\x_i^\prime}^\top\u_j}\right)
\end{equation}
given $\U$ a matrix of $\R^{d \times k}$ of unit-norm column $\u_j$.

\begin{algorithm}[t]
	\caption{Private and Smoothed Sliced Wasserstein Distance}
	\label{alg:dpswd}
	\begin{algorithmic}[1]
		\REQUIRE{ A public $\{\X_s\}$ and private $\{\X_t\}$ matrix both in $\R^{n \times d}$, $\sigma$ the standard deviation of a Gaussian distribution, $k$ the number of
		direction in SWD, $q$ the power in the SWD.}	  										\STATE \emph{// random projection}
		\STATE construct random projection matrix $\U \in \R^{d \times k}$ with unit-norm columns.
				\STATE construct two random  Gaussian, with standard deviation $\sigma$ noise, matrices $\V_s$ and $\V_t$ of size $n \times k$ 
										\STATE \emph{// Gaussian mechanism}
		\STATE  compute $\mathcal{M}(\X_s) = \X_s\U + \V_s$, $\mathcal{M}(\X_t) = \X_t\U + \V_t$
		\STATE $\text{DP}_\sigma\text{SWD}_q^q \leftarrow$ compute Equation \eqref{eq:approxswd} using $\mathcal{M}(\X_s)$~and $\mathcal{M}(\X_t)$ as the locations of the Diracs.
						\RETURN  $\text{DP}_\sigma\text{SWD}_q^q$
	\end{algorithmic}
\end{algorithm}
\section{Private and Smoothed Sliced Wasserstein Distance}
\label{sec:methods}

We now introduce how we obtain a differentially
private approximation of the Sliced Wasserstein Distance. To achieve this goal,
we take advantage of the intrinsic randomization process that is embedded in 
the Sliced Wasserstein distance.

\subsection{Sensitivity of Random Direction Projections}
In order to uncover its $(\epsilon,\delta)$-DP , we analyze the sensitivity
of the random direction projection in SWD. 
 Let us consider the matrix 
$\X \in \R^{n \times d}$ representing a dataset 
composed of $n$ examples in dimension $d$ organized in row (each sample being randomly drawn from the distribution $\mu$). 
One  mechanism of interest is 
$$
\mathcal{M}_u(\X) = \X \frac{\u}{\|\u\|_2} + \v.
$$
where $\v$ is a vector whose entries \rev{are} drawn from a zero-mean Gaussian distribution. 
Let $\X$ and $\X^\prime$  be two matrices in $\R^{n \times d}$ that differ
only on one row, say $i$ and such that $ \|\X_{i,:} - \X_{i,:}^\prime\|_2 \leq 1$,
where $\X_{i,:}\in\R^d$ and $\X_{i,:}^\prime\in\R^d$ are the $i$-th row of $\X$ 
and $\X^\prime$, respectively. For ease of notation, we will from now on use
$$\z\doteq(\X_{i,:} - \X_{i,:}^\prime)^{\top}.$$ 

\begin{lemma} 
	\label{lem:single_beta}
	Assume that $\z\in\R^d$ is a unit-norm vector and $\u\in\R^d$ a vector where each entry is drawn independently from $\mathcal{N}(0,\sigma_u^2)$. Then
				$$
 Y\doteq\Big(\z^\top \frac{\u}{\|\u\|_2}\Big)^2 \sim B(1/2, (d-1)/2)
$$
where $B(\alpha,\beta)$ is the beta distribution of parameters
$\alpha,\beta$. \end{lemma}
\begin{proof} See appendix.
\end{proof}
Instead of considering a randomized mechanism that projects only according to a single
random direction, \rev{we are interested in the whole set of projected (private) data
according to the random directions sampled through the Monte-Carlo approximation 
of the Sliced Wasserstein distance computation~\eqref{eq:approxswd}.} \rev{Our key interest
is therefore in the mechanism
$$
\mathcal{M}(\X) = \X \U + \V
$$
and in the sensitivity of $\X\U$.} Because of its randomness, we are
interested  in a probabilistic tail-bound of $\|\X \U - \X^\prime \U\|_F$, where the matrix $\U$ has columns independently  drawn from $\mathbb{S}^{d-1}$. 

\begin{lemma} 	\label{lem:several_beta}
	Let $\X$ and $\X^\prime$  be two matrices in $\R^{n \times d}$ that differ only in one row, and for that row, say $i$, 
	$\|\X_{i,:} - \X_{i,:}^\prime\|_2 \leq 1$. Denote $\U \in \R^{d \times k}$ and 
	 $\U$ has columns independently  drawn from $\mathbb{S}^{d-1}$. 
		With probability at least $1-\delta$, we have 
	\begin{align}
	\left\|\X \U - \X^\prime \U \right\|_F^2\leq w(k,\delta),
	\intertext{with}
	w(k,\delta)\doteq\frac{k}{d}+\frac{2}{3}\ln\frac{1}{\delta} + \frac{2}{d}\sqrt{k\frac{d-1}{d+2}\ln\frac{1}{\delta}}
	\end{align}
\end{lemma}
\begin{proof} See appendix.
\end{proof}

The above bound on the squared sensitivity has been obtained by first showing that the random variable $\left\|\X \U - \X^\prime \U \right\|_F^2$ is the sum of $k$ iid Beta-distributed random variables and then by using a Bernstein inequality. This bound, referred to as {\em the 
Bernstein bound,} is very conservative as soon as $\delta$ is very small. By calling the
Central Limit Theorem (CLT), assuming that $k$ is large enough ($k>30$), we get under the same hypotheses (proof is in the appendix)  that
$$
	w(k,\delta) = \frac{k}{d} + \frac{z_{1-\delta}}{d}\sqrt{\frac{2k(d-1)}{d+2}}
$$
where  $z_{1-\delta}=\Phi^{-1}(1-\delta)$ and $\Phi$ is the cumulative distribution function of a zero-mean unit variance Gaussian distribution. This bound is far tighter
but is not rigourous due to the CLT approximation. Figure~\ref{fig:examplebound}
presents an example of the probability distribution histogram of $ \|(\X - \X^\prime)^\top \U\|_F^2 = \|(\X_{i,:} - \X_{i,:}^\prime)^\top \U\|_2^2$
for two fixed arbitrary $\X_{i,:}$, $\X_{i,:}^\prime$ and  for $10000$ random draws of $\U$. \rev{It shows that the CLT bound is numerically far smaller than the Bernstein bound of Lemma \ref{lem:several_beta}.}
Then, using {the $\w(k,\delta)$-based bounds} 
jointly with the Gaussian
mechanism property gives us the following proposition.

\begin{figure}
	\begin{center}
	\includegraphics[width=7.5cm]{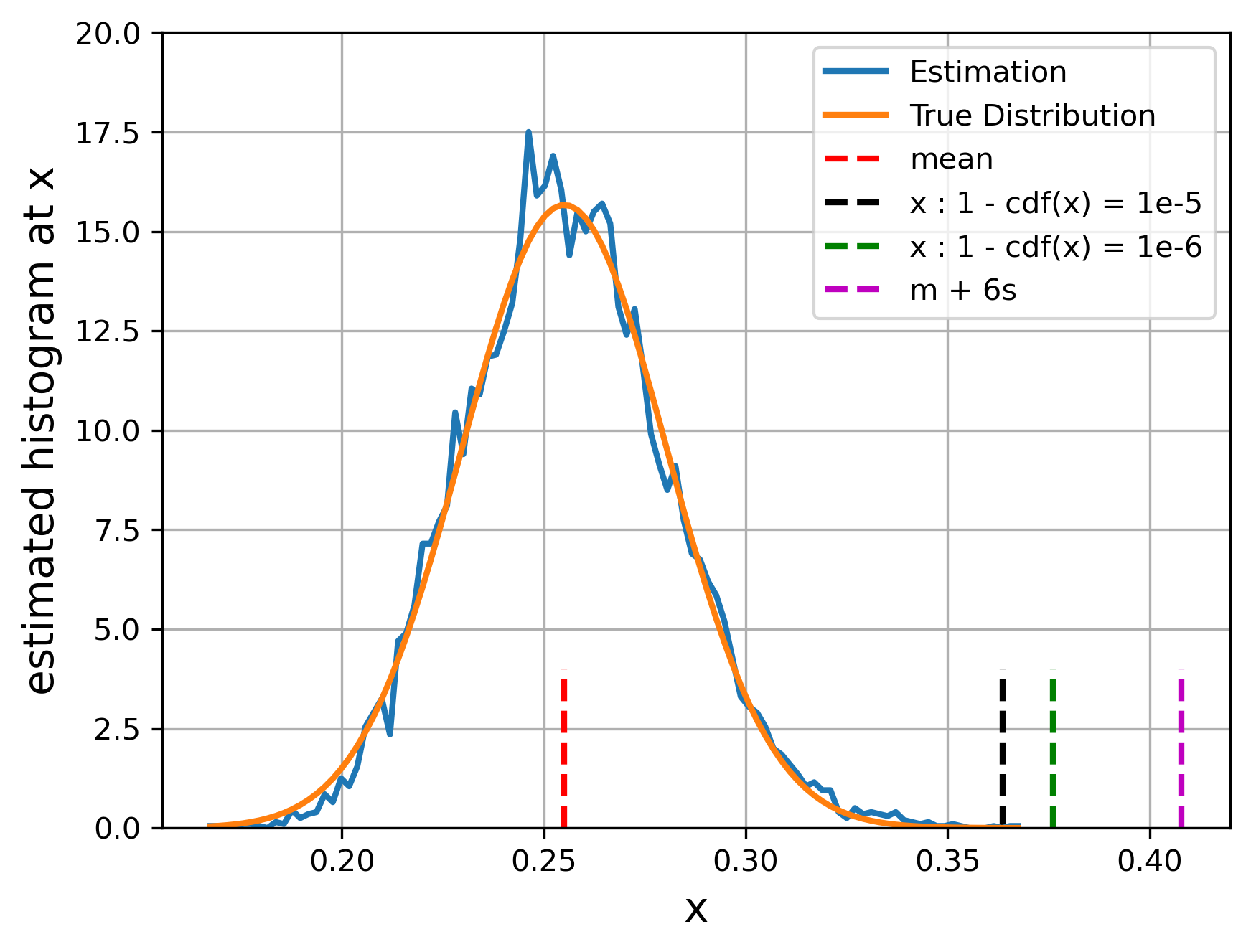}
\end{center}
\vspace{-0.5cm}
\caption{Estimated density probability of $\sum_i^k Y_i$ and the Normal distribution of same mean and standard deviation. Here, we have $k=200$ and $d=784$ which corresponds to the dimensionality of MNIST digits and  the number of random projections we use in the experiments. We illustrate also some  bounds (on the squared sensitivity) that can be derived from this Normal distribution as well as our CLT bound. Note that the Bernstein bound  is above 1 and
in this example, that the CLT bound, which is numerically equal to the inverse CDF of the Normal distribution at desired $\delta$.\label{fig:examplebound}}
\end{figure}

\begin{proposition}
	Let $\alpha>1$ and $\delta \in [0,1/2]$, given a random direction projection matrix $\U \in \R^{d\times k}$, then the Gaussian
	mechanism $\mathcal{M}(\X)= \X\U + \V$, where $\V$ is a Gaussian  matrix in $\R^{n \times k}$ with entries drawn from 
	$\mathcal{N}(0,\sigma^2)$ is $(\frac{\alpha {w(k,\delta/2)}}{2\sigma^2} + \frac{log(2/\delta)}{\alpha-1},\delta)$-DP.
\end{proposition}
\begin{proof}
The claim derives immediately by the relation between RDP and DP and by Lemma \ref{lem:several_beta}
with $\frac{\delta}{2}$.
\end{proof}

The above DP guarantees apply to the full dataset. Hence, when learning through mini-batches, we benefit from
the so-called privacy amplification
by the ``subsampling'' principle, which ensures that a differentially private mechanism
run on a random subsample of a population \rev{leads to}  higher privacy guarantees
than when run on the full population \cite{balle2018neurips}.  On the contrary,
gradient clipping/sanitization acts individually one each gradient and  thus do not fully benefit from the subsampling amplification, as its DP property may still depend on the batch size \cite{neurips20chen}.

\rev{This Gaussian mechanism on the random direction projections $\mathcal{M}(\X)$ can be related to the definition of the empirical SWD as each  $\x_i^\top \u_j$ corresponds to one entry of $\X\U$. Hence, by adding a Gaussian noise to each projection, we naturally derive our empirical  DP Sliced Wasserstein distance, which inherits the differential property of $\mathcal{M}(\X)$, owing to the post-processing proposition \cite{dwork2014algorithmic}.
}

\subsection{Metric Properties of DP-SWD}

We have analyzed the sensitivity of the random direction projection central to SWD
and we have proposed  a Gaussian mechanism {to obtain}
 a differentially private SWD (DP-SWD) which steps
are \rev{depicted}  in Algorithm \ref{alg:dpswd}.
In our use-cases, DP-SWD is {used}  in a context of learning to match two distributions (one of them {requiring}  to be privately protected).  Hence, the utility guarantees of our DP-SWD  is more related to the ability of the mechanism to distinguish two different distributions rather than on the equivalence between SWD and DP-SWD.
Our goal in this section is to investigate the impact of 
adding Gaussian noise to the source $\mu$ and target
$\nu$ distributions in terms of distance property in the population case.

Since $\mathcal{R}_\u$, as defined in Equation \rev{\eqref{eq:radon}}, is a push-forward operator of
probability distributions, the Gaussian mechanism process
implies that the Wasserstein distance involved in SWD compares
two 1D probability distributions which are respectively the convolution of a Gaussian distribution and  $\mathcal{R}_\u \mu$ and	 $\mathcal{R}_\u \nu$. Hence, we can consider DP-SWD uses as a building
block the 1D smoothed Wasserstein distance between  $\mathcal{R}_\u \mu$ and	 $\mathcal{R}_\u \nu$ with the smoothing being ensured by 
$\mathcal{N}_\sigma$ and its formal definition \rev{being}, for $q \geq 1$,
$$
\text{DP}_\sigma\text{SWD}_q^q(\mu,\nu) \doteq \int_{\mathbb{S}^{d-1}}  W_q^q(\mathcal{R}_\u \mu * \mathcal{N}_\sigma,\mathcal{R}_\u \nu * \mathcal{N}_\sigma ) u_d(\u)d\u 
$$

While some works have analyzed the theoretical properties of the Smoothed Wasserstein distance \cite{goldfeld2020asymptotic,goldfeld2020gaussian}, as far as we know, no theoretical result is available for the smoothed Sliced Wasserstein distance, and we provide in the sequel some \rev{insights} that help its understanding. The following property shows that DP-SWD preserves the identity of indiscernibles.
\begin{property} For continuous probability distributions $\mu$ and $\nu$, we have, for $q \geq 1$,
	$\text{DP}_\sigma\text{SWD}_q^q(\mu,\nu) =  0 \Leftrightarrow \mu=\nu$ $\forall\sigma > 0$.
\end{property} 
\begin{proof}
	Showing that $\mu=\nu \implies \text{DP}_\sigma\text{SWD}_q^q(\mu,\nu) =  0
	$ is trivial as the Radon transform and the convolution are two well-defined maps. We essentially would like to show that $\text{DP}_\sigma\text{SWD}_q^q(\mu,\nu) =  0$ implies $\mu=\nu$.
	If $\text{DP}_\sigma\text{SWD}_q^q(\mu,\nu) =  0$ then
	$\mathcal{R}_\u \mu * \mathcal{N}_\sigma = \mathcal{R}_\u \nu * \mathcal{N}_\sigma$ for almost every $\u \in {\mathbb{S}^{d-1}}$. 
							\rev{As convolution yields to multiplication in the Fourier domain and
	because, the Fourier transform of a Gaussian is also a Gaussian and thus is always positive, 
	one can show that we have for all $\u$ equality of the
	Fourier transforms of $\mathcal{R}_\u \mu$ and $\mathcal{R}_\u \nu$.
					Then, owing to the continuity of $\mu$ and $\nu$ and by the Fourier inversion theorem, we have	$\mathcal{R}_\u \mu = \mathcal{R}_\u \nu$.
	Finally, as for the SWD proof \citep[Prop 5.1.2]{bonnotte2013unidimensional}, this implies that $\mu = \nu$, 	owing to the projection nature of the Radon Transform
	and  because the Fourier transform is injective.}
	\end{proof}

\begin{property}
	For $q \geq 1$, $\text{DP}_\sigma\text{SWD}_q^q(\mu,\nu)$ is symmetric and satisfies the triangle inequality.
\end{property}
\begin{proof} The proof easily derives from the metric properties of Smoothed Wasserstein distance \cite{goldfeld2020gaussian} and details are in the appendix.	
\end{proof}

These properties are strongly relevant in the context of our machine learning 
applications. Indeed, while they do not tell us how the value of DP-SWD compares with  SWD, at fixed $\sigma>0$ or when $\sigma \rightarrow 0$, they show that they can properly act as (for any $\sigma >0$) loss functions  to minimize if we aim to match distributions (at least in the population case). Naturally, there are still several theoretical properties of $\text{DP}_\sigma\text{SWD}_q^q$ that are worth investigating but that are beyond the scope of this work.

\section{DP-Distribution Matching Problems}
\label{sec:models}

\begin{algorithm}[t]
	\caption{Differentially private DANN with DP-SWD }
	\label{alg:DPSWD}
	\begin{algorithmic}[1]
		\REQUIRE{ $\{\X_s,\y_s\}, \{\X_t\}$, respectively the public and private domain, $\sigma$ standard deviation of the Gaussian mechanism}
				\STATE Initialize representation mapping $g$, the classifier $h$  with parameters $\theta_g$, $\theta_h$
		\REPEAT	

		\STATE sample minibatches $\{x_B^s,y_B^s\}$ from $\{x_i^s,y_i^s\}$
		\STATE {compute $g(x_B^s)$}  
		\STATE compute the  classification loss ${L}_c = \sum_{i \in B}  L(y_i^s,h(g(x_i^s)))$
		\STATE $\theta_h \leftarrow \theta_h - \alpha_h \nabla_{\theta_h} {L}_c  $
		\STATE {\emph{//~Private steps : $g(x_B^t)$ is computed in a private way. $g(\cdot)$ is either transferred or has shared weights between public and private clients. }}
		\STATE sample minibatches $\{x_B^t\}$ from $\{x_B^t\}$
		\STATE {compute $g(x_B^t)$}  
		\STATE normalize each sample $g(x_{B}^s)$ wrt $2\max_j{\|g(x_{B,j}^s)\|_2}$
		\STATE normalize each sample $g(x_{B}^t)$ wrt $2\max_j{\|g(x_{B,j}^t)\|_2}$
		\STATE {compute $\text{DP}_\sigma\text{SWD}(g(x_B^s), g(x_B^t))$}
		\STATE {publish $\nabla_{\theta_g} \text{DP}_\sigma\text{SWD}$}
		\STATE {{//\emph{~public step}}}
		\STATE \color{black}$\theta_g \leftarrow \theta_g - \alpha_g \nabla_{\theta_g} L_c - \alpha_g\nabla_{\theta_g} \text{DP}_\sigma\text{SWD}$
		\UNTIL{a convergence condition is met} 
	\end{algorithmic}
\end{algorithm}

There exists several machine learning problems where distance between distributions
is the key part of the loss function to optimize.
In domain adaptation, one learns a classifier from  public source dataset  but looks to adapt it to private target dataset (target domain examples are available only through a privacy-preserving mechanism). In generative modelling,
the goal is to generate samples similar to true data which are 
accessible only through a privacy-preserving mechanism.
In the sequel, we describe how our $\text{DP}_\sigma\text{SWD}_q^q$ distance can be instantiated into these two learning paradigms for measuring adaptation or
for measuring similarity between generated an true samples. 

For unsupervised domain adaptation, given source examples $\X_s$ and their label $\y_s$ and unlabeled
	private target examples $\X_t$, the goal is to learn a classifier $h(\cdot)$ trained on 
	the source examples that generalizes well on the target ones. One usual 
	technique is to learn a representation mapping  $g(\cdot)$ that leads to invariant latent representations, invariance being measured as \rev{some} distance between empirical distributions of mapped source and target samples. Formally, this leads to the following learning problem
	\begin{equation}
		\min_{g,h} L_c(h(g(\X_s)),\y_s) + \text{DP}_\sigma\text{SWD}(g(\X_s), g(\X_t))
	\end{equation}
	where $L_c$ can be any loss function of interest and
	$\text{DP}_\sigma\text{SWD}=\text{DP}_\sigma\text{SWD}_q$. We solve this problem through stochastic gradient descent, similarly to many approaches that use Sliced Wasserstein Distance as a distribution distance \cite{lee2019sliced}, except that in our case, the gradient of 	$\text{DP}_\sigma\text{SWD}$ involving the target dataset  is $(\varepsilon,\delta)$-DP.
	Note that in order to compute the 	$\text{DP}_\sigma\text{SWD}$, one needs the public dataset $\X_s$
	and  the public generator. In practice, this generator can either be
	transferred, after each update, from the private client curating $\X_t$ or
	can be duplicated on that client. 	The resulting algorithm is presented in Algorithm \ref{alg:DPSWD}.

	In the context of generative modeling, we follow the same steps as \citet{deshpande2018generative}  but use our $\text{DP}_\sigma\text{SWD}$ instead of SWD.
				Assuming that we have some examples of data $\X_t$ sampled from a
	given distribution, the goal of the learning problem is to learn a generator
	$g(\cdot)$ to  output samples similar to those of the target distribution, with at its input a given noise vector.	This is usually achieved by solving 
	\begin{equation}
	\min_g \text{DP}_\sigma\text{SWD}(\X_t, g(z))
	\end{equation}  
	where $z$ {is for instance
	a Gaussian vector.}
	 In practice, we solve this problem using a mini-batching stochastic gradient descent
	strategy, following a similar algorithm than the one for domain adaptation. 
	The main difference is that the private target dataset does not pass through the generator.
	 
	\paragraph{Tracking the privacy loss}
	
	Given that we consider the privacy mechanism within a stochastic gradient descent 	framework, we keep track of the privacy loss through the RDP accountant  proposed by \citet{wang2019subsampled} 	for composing subsampled private mechanisms. Hence,
	we used the PyTorch package \cite{xiang2020} that they made available for estimating the noise standard deviation $\sigma$ given the $(\varepsilon,\delta)$ budget, a number of epoch, a fixed batch size, the number
	of private samples, the dimension $d$  of the distributions to be compared and the
	number $k$ of projections used for $\text{DP}_\sigma\text{SWD}$. Some examples of Gaussian noise standard deviation are reported in Table \ref{tab:parameter} in the appendix.

\section{Related Works}
\label{sec:related}

\subsection{DP Generative Models} Most recent approaches \cite{fan2020survey} that proposed DP generative
models considered it from a GAN perspective and applied DP-SGD \cite{abadi2016deep} for training the model. The main idea for introducing privacy is to appropriately clip the gradient and to add calibrated noise into the model's parameter gradient during training \cite{torkzadehmahani2019dp, neurips20chen, xie2018differentially}. This added noise make those models
even harder to train. Furthermore, since the DP mechanism applies to each single gradient, those approaches do not fully benefit from the amplification induced by subsampling (mini-batching) mechanism
\cite{balle2018neurips}. The work of \citet{neurips20chen} uses
gradient sanitization and achieves privacy amplification by training multiple
discriminators, as in \citep{jordon2018pate}, and sampling on them for adversarial training. While their approach is 
competitive in term of quality of generated data, it is hardly tractable for large
scale dataset, due to the multiple (up to 1000 in their experiments) discriminator trainings.

Instead of considering adversarial training, some DP generative model works have investigated the use
of distance on distributions. \citet{harder2020differentially} proposed random feature based maximum-mean embedding distance for computing distance between empirical distributions. \citet{cao2021differentially} considered the Sinkhorn divergence for
computing distance between true and generated data and used gradient clipping and noise addition for privacy preservation. Their approach is then very similar to
DP-SGD in the privacy mechanism.
Instead, we perturb the Sliced Wasserstein distance by smoothing the distributions to compare. This yields a privacy mechanism that benefits
subsampling amplification, as its sensitivity does not depend on the number of samples, and that preserves its utility as
the smoothed Sliced Wasserstein distance is still a distance.

\subsection{Differential Privacy with Random Projections}
Sliced Wasserstein Distance leverages on Radon transform for mapping
high-dimensional distributions into 1D distributions. This is related
to projection on random directions and the sensitivity analysis
of those projections on unit-norm random vector is key.
The first use of random projection for differential privacy 
has been introduced by \citet{kenthapadi2012privacy}. Their approach
was linked to the distance preserving property of random projections
induced by the Johnson-Lindenstrauss Lemma. As a natural extension, \citet{letien2019differentially} and \citet{gondara2020differentially}
have applied this idea in the context of optimal transport and
classification. The fact that we project on unit-norm random vector, instead of any random vector as in \citet{kenthapadi2012privacy},
requires a novel sensitivity analysis and we show that
this sensitivity scales gracefully with ratio of the number of
projections and dimension of the distributions.

\section{Numerical Experiments}
\label{sec:expe}
In this section, we provide some numerical results showing how our differentially private Sliced Wasserstein Distance works in practice.
The code for reproducing some of the results is
available in \url{https://github.com/arakotom/dp_swd}.

\subsection{Toy Experiment}
The goal of this experiment is to illustrate the behaviour of 
the DP-SWD compared with the  SWD in controlled situations.
We consider the source and target distributions as isotropic Normal distributions of unit variance with added \rev{privacy-inducing} Gaussian noise of different  variances. Both distributions are Gaussian of dimension $5$ and the means of the source and target are respectively $\m_\mu=0$ and $\m_\nu = c 1$ with $c \in [0,1]$.
Figure \ref{fig:toy} presents the evolution of the distances averaged over $5$ random draws of the Gaussian and noise.
 When source and target distributions are different, this experiment shows that DP-SWD follows the same  increasing trend as SWD. This suggests that the order relation between distributions as evaluated using SWD is preserved by DP-SWD, and that the distance DP-SWD is minimized when
 $\mu=\nu$,  which are important features when using DP-SWD as a loss.

\begin{figure}
			\includegraphics[width=4cm]{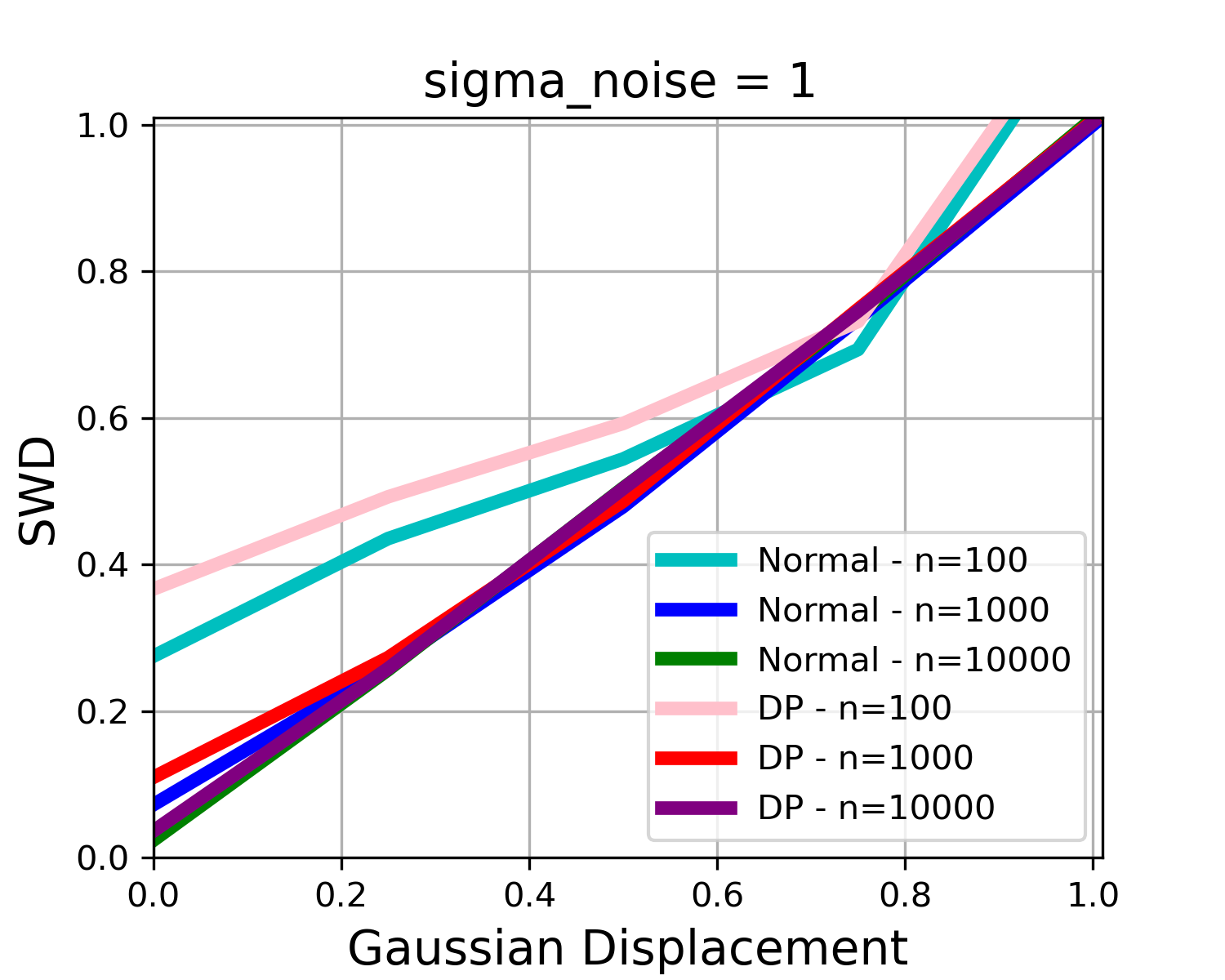}
	\includegraphics[width=4cm]{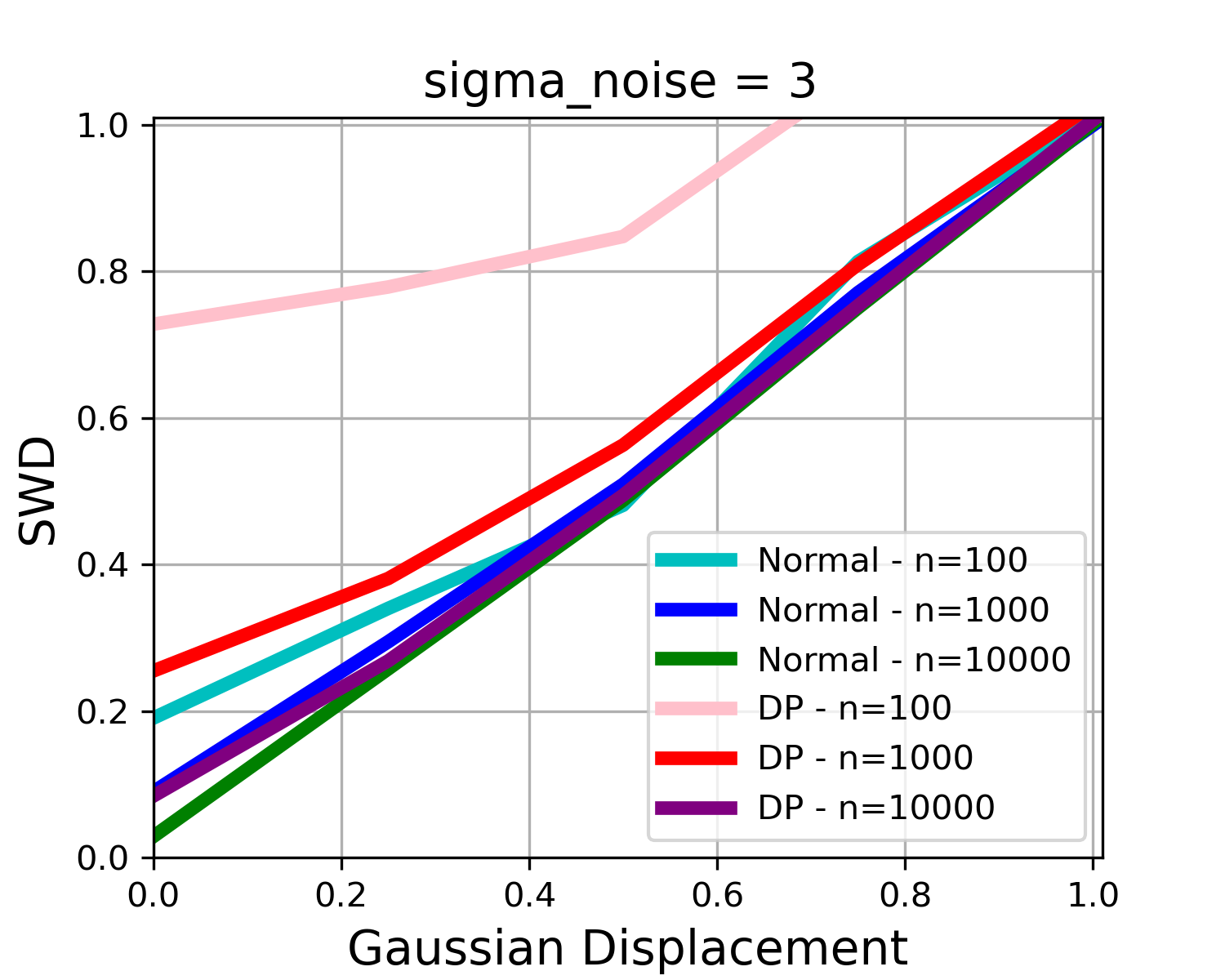}
	\caption{Comparing  SWD and  DP-SWD by measuring the distance between two normal distributions (averaged over $5$ draws of all samples). 	The comparison holds when the distance between the means of the Gaussians increases linearly, for different noise amplitudes of the Gaussian mechanism and 	different number of samples. (left) $\sigma=1$. (right) $\sigma=3$. \label{fig:toy}}
\end{figure}
\subsection{Domain Adaptation}
\begin{figure}
\begin{center}
\includegraphics[width=7cm]{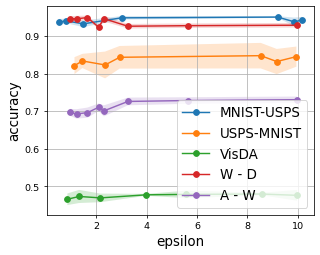}
\end{center}
\vspace{-0.5cm}
\caption{Evolution of the target domain accuracy in UDA  with respect to the $\varepsilon$ parameter for fixed  value of $\delta$, for $3$ different datasets. Sensitivity of DP-SWD has been computed using the Bernstein bound.\label{fig:epsacc}}
\end{figure}
 We conduct experiments  for evaluating our DP-SWD distance in the context of classical unsupervised domain adaptation (UDA) problems such
as handwritten digit recognitions (MNIST/USPS), synthetic to real object data
(VisDA 2017) and Office 31 datasets. Our goal is to analyze how DP-SWD  performs compared with its public counterpart SWD \cite{lee2019sliced}, with one DP deep domain adaptation algorithm DP-DANN that is based on gradient clipping \cite{wang2020} and with the classical non-private DANN algorithm.
 Note that we need not compare with \cite{letien2019differentially} as their algorithm does not learn representation and does not handle large-scale problems, as the OT transport matrix coupling need be computed on the full dataset. For all methods and for each dataset, we used
the same neural network architecture for representation mapping and for
classification. Approaches differ only on how distance between distributions have been computed.
Details of problem configurations as well as model architecture and training procedure
can be found in the appendix. Sensitivity has been computed using the Bernstein bound
of Lemma~\ref{lem:several_beta}.

Table \ref{tab:da} presents the accuracy on the target domain for all methods averaged over $10$ iterations. We remark that our private model outperforms the DP-DANN approach on all problems  except on two difficult ones. Interestingly, our method does not incur a loss of performance despite the private mechanism. This finding is confirmed in Figure~\ref{fig:epsacc} where we plot the performance of the model with respect to the noise level $\sigma$ (and thus the privacy parameter $\varepsilon$). Our model is able to keep accuracy
almost constant for $\varepsilon \in [3,10]$.

\begin{table}[t]
	\begin{center}
	\caption{Table of accuracy on the private target domain for different domain adaptation problems M-U, U-M refers to MNIST-USPS and USPS-MNIST  the first listed data being the source domain. (D,W,A) refers to domains in the Office31 dataset. For all the problems, $\varepsilon=10$ and $\delta$ depends on the 	number of examples in target domain.  $\delta$ has been respectively set
	to $10^{-3}$,$10^{-5}$,$10^{-6}$ for Office31, MNIST-USPS and VisDA.\label{tab:da}}
	\begin{tabular}{l|ll|ll}
\toprule
 & \multicolumn{4}{c}{Methods} \\
 Data & DANN & SWD & DP-DANN & DP-SWD\\\midrule
M-U &  93.9 $\pm$  0  & 95.5 $\pm$  1  & 87.1 $\pm$  2 &  \textbf{94.0$\pm$ 0}  \\
U-M &  86.2 $\pm$  2   & 84.8$\pm$2  & 73.5 $\pm$  2 &	\textbf{83.4$\pm$ 2}	\\
VisDA & 57.4 $\pm$  1  & 53.8$\pm$1 & \textbf{49.0 $\pm$ 1}& 47.0$\pm$ 1 \\ 
D - W & 90.9$\pm$ 1  & 90.7$\pm$ 1  & 88.0$\pm$ 1 &\textbf{90.9$\pm$ 1}  \\
D - A & 58.6$\pm$ 1  & 59.4$\pm$ 1  & \textbf{56.5$\pm$ 1} & 55.2$\pm$ 2  \\
A - W  & 70.4$\pm$ 3  & 74.5$\pm$ 1  & 68.7$\pm$ 1 & \textbf{72.6$\pm$ 1} \\
A - D  & 78.6$\pm$ 2  & 78.5$\pm$ 1  & 73.7$\pm$ 1 &\textbf{79.8$\pm$ 1} \\
W - A  & 54.7$\pm$ 3  & 59.1$\pm$ 0  & 56.0$\pm$ 1 &\textbf{59.0$\pm$ 1} \\
W - D & 91.1$\pm$ 0  & 95.7$\pm$ 1  & 63.4$\pm$ 3 &\textbf{92.6$\pm$ 1} \\
\bottomrule
	\end{tabular}
\end{center}
\end{table}

\begin{figure}[t]
\begin{center}
	\includegraphics[width=6cm]{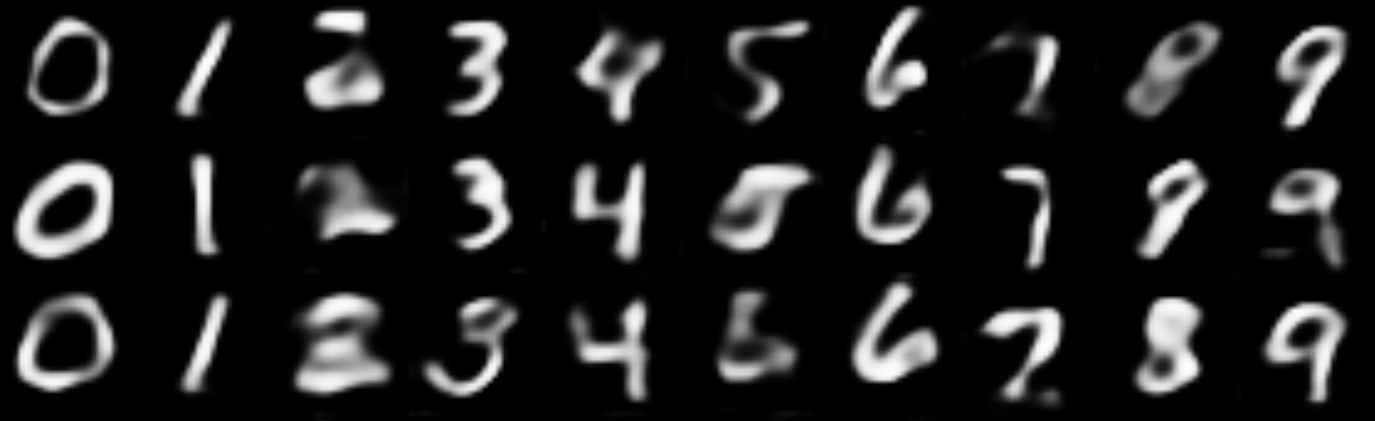}~\\~\\	\includegraphics[width=6cm]{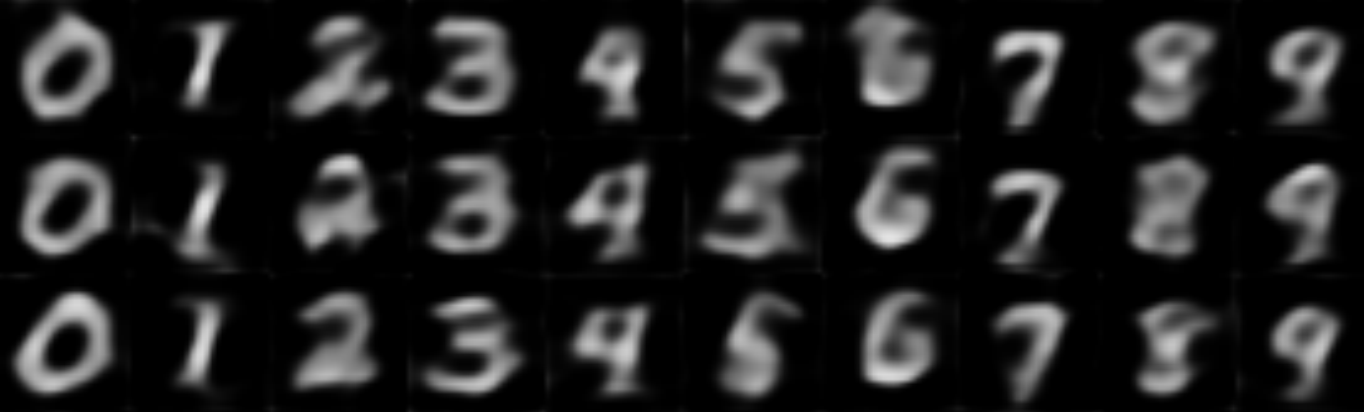}~\\~\\
		\includegraphics[width=6cm]{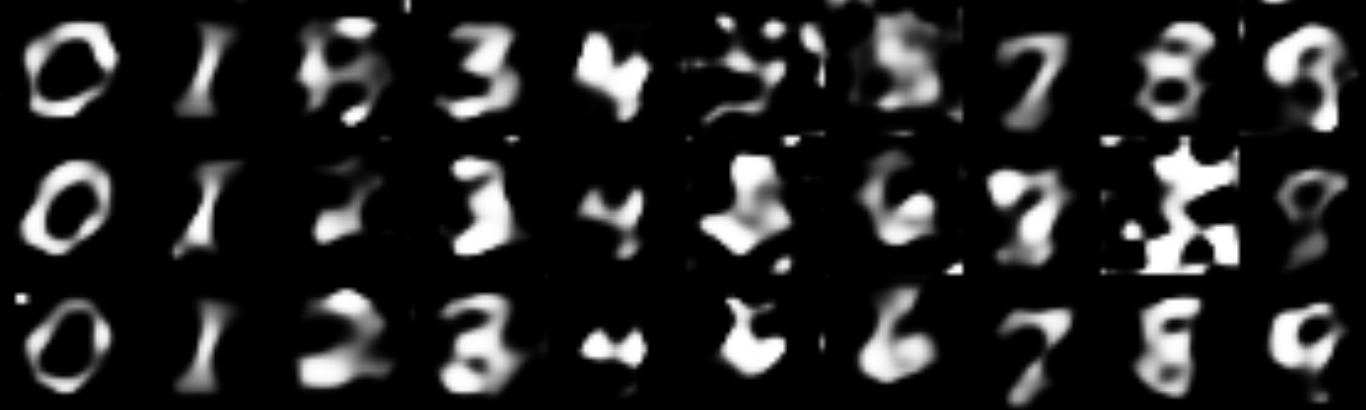}
		\caption{Examples of generate MNIST samples from (top) non-private SWD (middle) DP-SWD-b (bottom) MERF. \label{fig:mnistdigits}}
\end{center}
\end{figure}
\begin{figure}[t]
	\begin{center}
		\includegraphics[width=6cm]{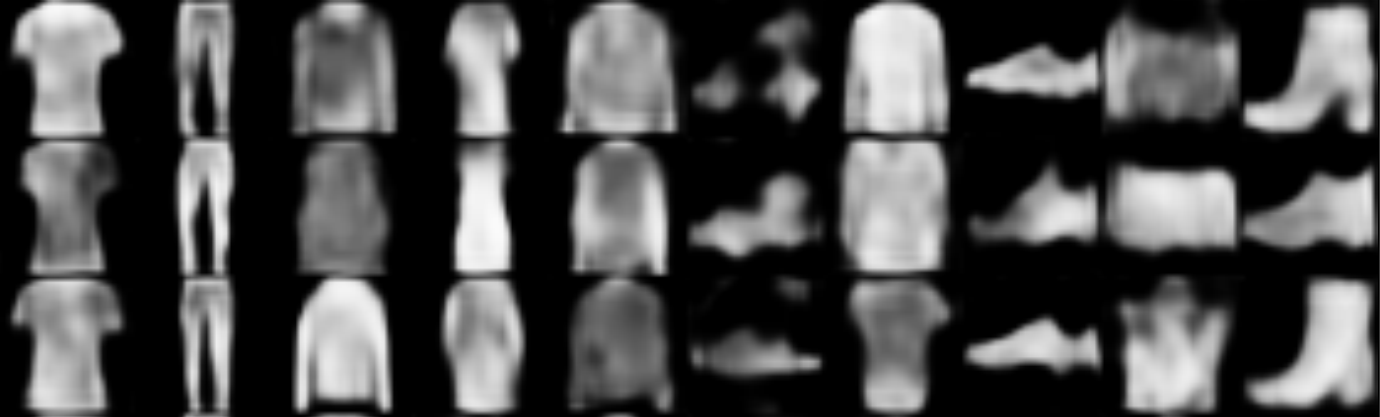}~\\~\\	\includegraphics[width=6cm]{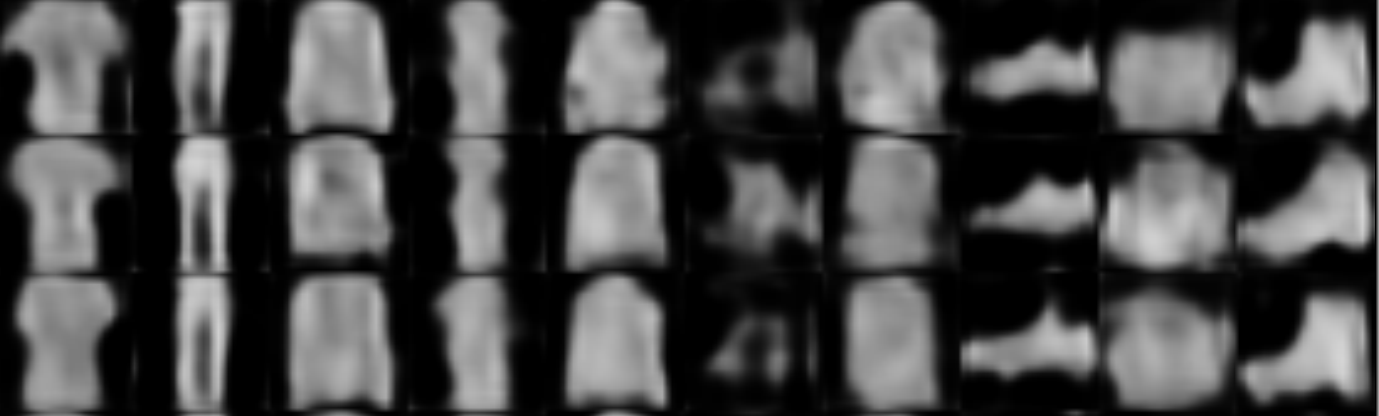}~\\~\\
		\includegraphics[width=6cm]{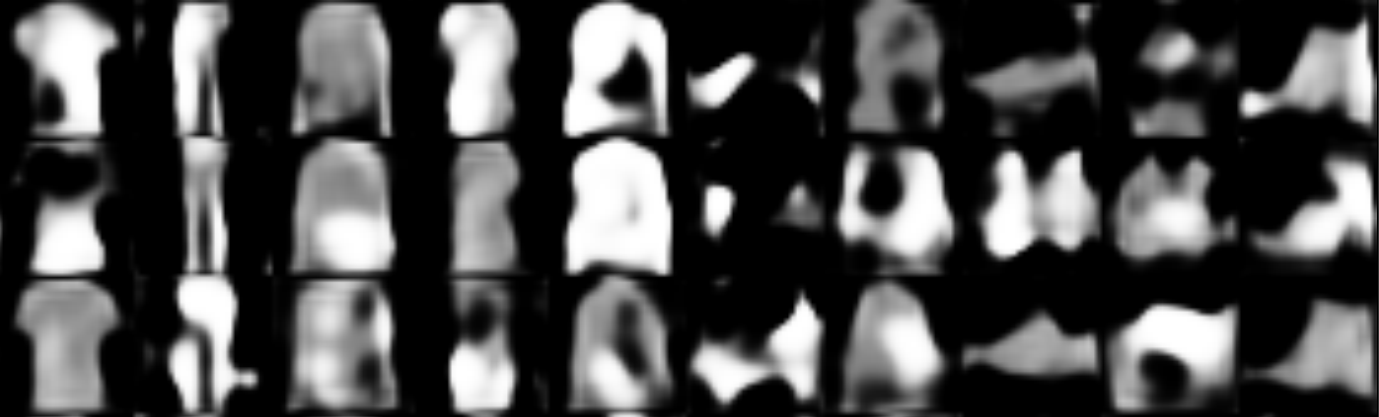}
		\caption{Examples of generate FashionMNIST samples from (top) non-private SWD (middle) DP-SWD-b (bottom) MERF. \label{fig:mnistfashion}}
	\end{center}
\end{figure}

\subsection{Generative Models}

\begin{table}
	\begin{center}
		\caption{Comparison of DP generative models on MNIST and FashionMNIST at privacy level $(\varepsilon,\delta)=(10,10^{-5})$. The downstream task is a $10$-class classification problems using the synthetic generated dataset. We report the accuracy of different classifiers. Results are averaged over $5$ runs of generation. SWD is the non-private version of our generative model.\label{tab:genmnist}}
		\begin{tabular}{llllll}
			\toprule
			& \multicolumn{2}{c}{MNIST} & \multicolumn{2}{c}{FashionMNIST} \\\midrule
			Method &  MLP  & LogReg &  MLP  & LogReg \\\midrule
			SWD &  87  & 82    &  77  & 76    \\ 
			GS-WGAN& 79&79& 65 & 68\\  \midrule
			DP-CGAN& 60&60&  50 & 51\\ 	
			DP-MERF& 76&75&  \textbf{72}& \textbf{71}\\ 
			DP-SWD-c& \textbf{77}  & \textbf{78}    & 67  & 66  \\
			DP-SWD-b & 76  & 77    &  67  & 66  \\\bottomrule
		\end{tabular}
	\end{center}
\end{table}
In the context of generative models, our first task is to generate synthetic samples for MNIST and Fashion MNIST dataset that will be afterwards used for learning a classifier. 
We compare with different gradient-sanitization strategies like DP-CGAN \cite{torkzadehmahani2019dp}, and GS-WGAN \cite{neurips20chen} and a model MERF \cite{harder2020differentially} that uses MMD as distribution distance. We report
results for our DP-SWD using two ways for computing the sensitivity,
by using the CLT bound and the Bernstein bound, respectively noted as
DP-SWD-c and DP-SWD-b.
 All models are compared with the same fixed budget of privacy $(\varepsilon,\delta)=(10,10^{-5})$. 
Our implementation is based on the one of MERF \cite{harder2020differentially}, in which we just plugged our DP-SWD loss in place of the MMD loss. The architecture of ours and MERF's generative model is composed
of few layers of convolutional neural networks and upsampling layers with approximately 180K parameters while the one of GS-WGAN is based on a ResNet with about 4M parameters. MERF's and our models have been trained over $100$ epochs with an Adam optimizer and batch size of $100$. For our DP-SWD we have used $1000$ random projections and the output dimension is the classical $28\times 28= 784$. 

Table \ref{tab:genmnist} reports some quantitative results on the task. We note that MERF and our DP-SWD perform on par on these problems (with a slight advantage for MERF
on FashionMNIST and for DP-SWD on MNIST). Note that our results on MERF are
better than those reported in \cite{neurips20chen}. We also remark that GS-WGAN performs the best at the expense of a model with $20$-fold more parameters and a very expensive training time (few hours just for training the $1000$ discriminators, while our model and MERF's take less than 10min).
Figure \ref{fig:mnistdigits} and \ref{fig:mnistfashion} present some examples of generated samples for MNIST and FashionMNIST. We can note that the samples generated by
DP-SWD present some diversity and are visually more relevant than those of MERF, although they do not lead to  better performance in the classification task. Our samples are a bit blurry compared to the ones generated by the non-private  SWD: this is an expected effect of smoothing.

We also evaluate our DP-SWD distance for training generative models on large RGB datasets
such as the $64 \times 64 \times 3$ CelebA dataset. To the best of our knowledge, no DP generative approaches have been experimented on such a dataset. For instance, the GS-WGAN of \cite{neurips20chen} has been evaluated only on grayscale MNIST-like problems. 
For training the model, we followed the same approach (architecture and optimizer) as the one described in \citet{nguyen2020distributional}. In that work, in order to reduce the dimension of the problems, distributions are compared in a latent space of dimension $d=8192$. We have used $k=2000$ projections which leads to a ratio $\frac{k}{d} <0.25$. Noise variance $\sigma$ and privacy loss over $100$ iterations have been evaluated using the PyTorch  package of \cite{wang2019subsampled} and have been calibrated for $\epsilon=10$ and $\delta=10^{-6}$, since the number of training samples is of the order of $170$K. Details are in the appendix. Figure \ref{fig:celeba} presents some examples of samples generated from our DP-SWD and SWD. We note that in this high-dimensional context, the sensitivity bound plays a key role, as
we get a FID score of 97 vs 158  respectively using CLT bound and Bernstein bound, the former being smaller than the latter.

\begin{figure}
	\centering
	\caption{Images generated on CelebA dataset. From top to bottom. Non-private SWD, DP-SWD with noise calibrated according to Gaussian approximation (CLT bound), DP-SWD with noise calibrated according to the Bernstein bound. The FID score computed over $10000$ generated examples of this three models are respectively $58$, $97$ and $149$. \label{fig:celeba}}~\\
	~\hfill \includegraphics[width=7.2cm]{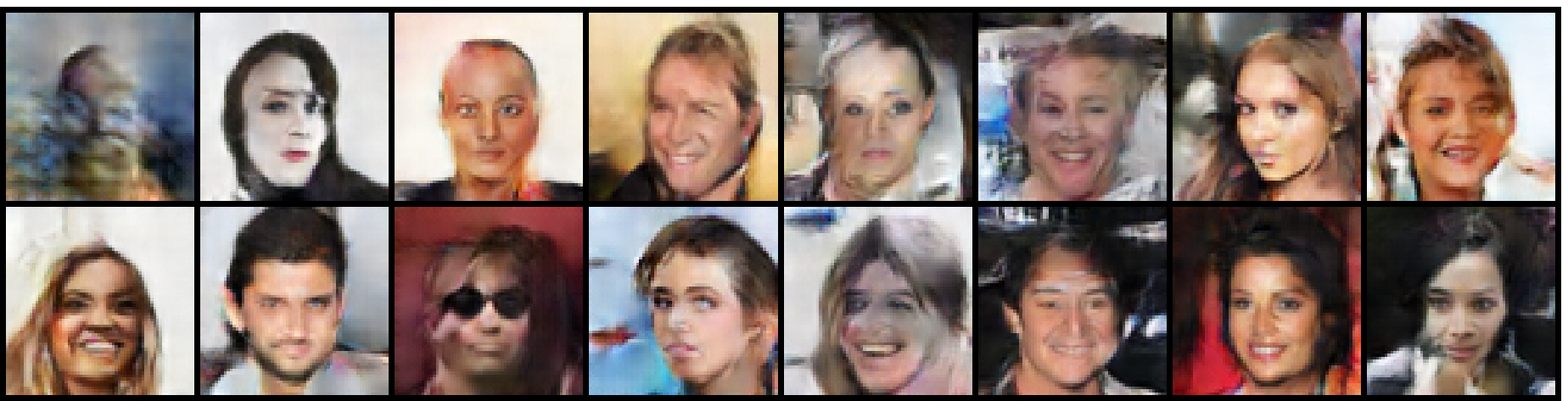} \hfill~\\~\\
		~\hfill
	\includegraphics[width=7.2cm]{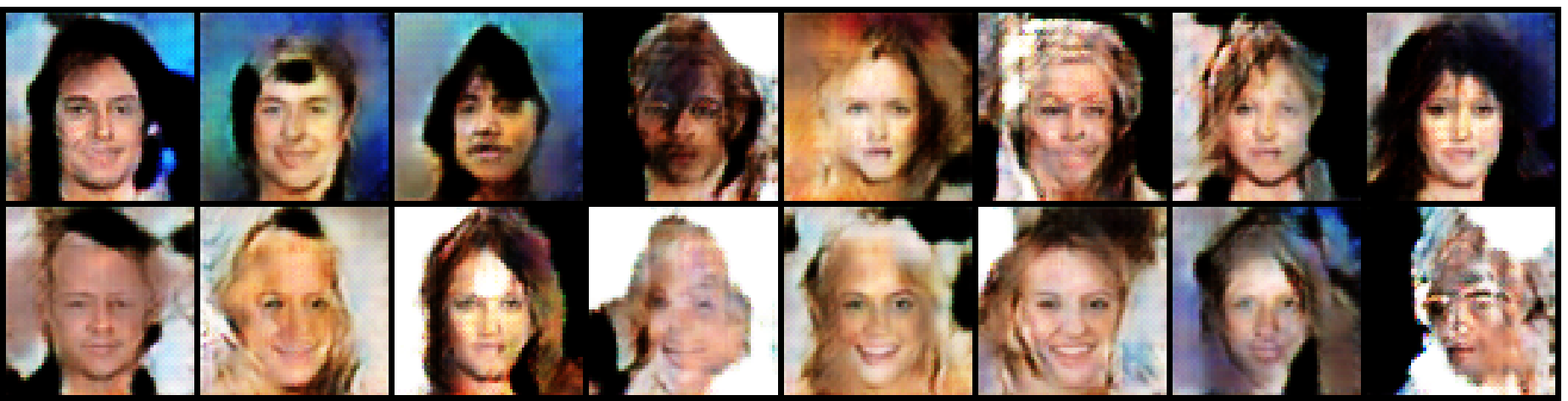} \hfill~\\~\\
		~\hfill
\includegraphics[width=7.2cm]{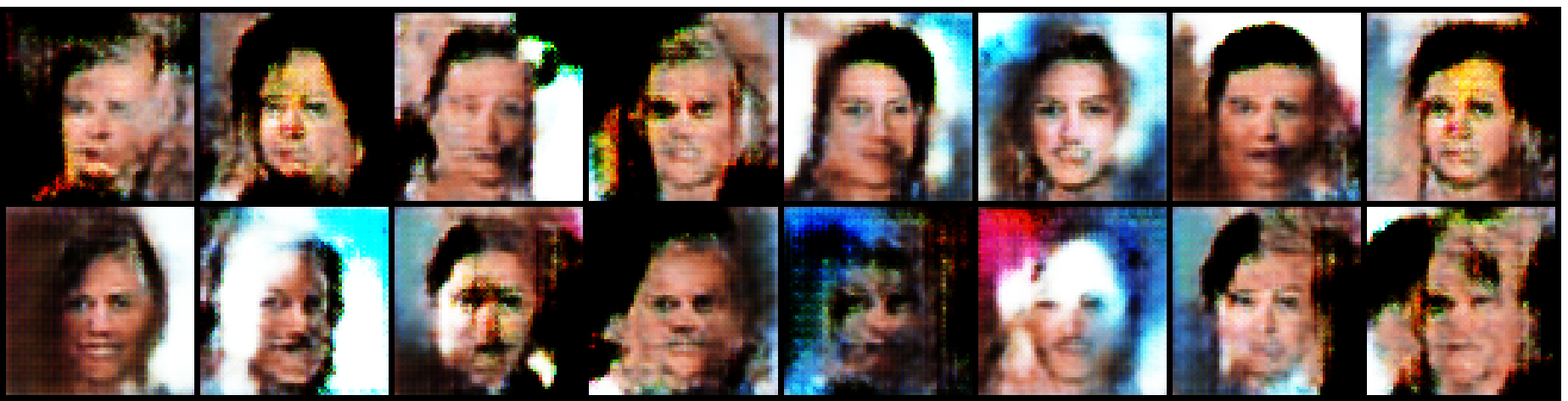} \hfill~
\end{figure}

\section{Conclusion}

This paper presents a differentially private distance on distributions based on the sliced Wasserstein distance. We applied a Gaussian mechanism on the random
projection inherent to SWD and analyzed its properties. We proved that 
a bound (à la Bernstein) on sensitivity of the mechanism as an inverse dependence on the problem dimension and that a Central limit theorem bound, although
approximate, gives a tighter bound. One of our key findings is that the privacy-inducing mechanism we proposed turns the SWD into a smoothed sliced Wasserstein distance, which inherits all the properties of a distance. Hence, our 
privacy-preserving distance can be plugged seamlessly into domain adaptation
or generative model algorithms \rev{to give effective privacy-preserving
learning procedures.}

\providecommand{\CH}{{C.-H}}\providecommand{\JB}{{J.-B}}

\bibliographystyle{icml2021} \clearpage
\newpage
\section*{\centering Supplementary material \\
Differentially Private Sliced Wasserstein Distance}

\section{Appendix}

\setcounter{lemma}{0}
\subsection{Lemma 1 and its proof}
\begin{lemma} 
		Assume that $\z\in\R^d$ is a unit-norm vector and $\u\in\R^d$ a vector where each entry is drawn independently from $\mathcal{N}(0,\sigma_u^2)$. Then
					$$
	Y\doteq\Big(\z^\top \frac{\u}{\|\u\|_2}\Big)^2 \sim B(1/2, (d-1)/2)
	$$
	where $B(\alpha,\beta)$ is the beta distribution of parameters
	$\alpha,\beta$. \end{lemma}
\begin{proof}
At first, consider a vector of unit-length in $\R^d$, say $\e_1$, that can
be completed to an orthogonal basis. A change of basis from the canonical one does not change the length of a vector as the transformation is orthogonal. Thus the distribution
of 
$$
 \frac{(\e_1^\top\u)^2}{\|\u\|_2^2} = \frac{(\e_1^\top\u)^2}{\sum_i^d u_i^2}
$$
does not depend on $\e_1$. $\e_1$ can be either the vector $(1,0,\cdots,0)$ in $\R^d$ or $ \z$ (as $\z$ is an unit-norm vector). However, for simplicity, let us consider $\e_1$ as $(1,0,\cdots,0)$, we thus have
 $$
 \frac{(\e_1^\top\u)^2}{\|\u\|_2^2} = \frac{u_1^2}{\sum_i^d u_i^2}
 $$
 where the $u_i$ are iid from a normal distribution of standard deviation
 $\sigma_u$. Hence, because $u_1$ and the $\{u_i\}_{i=2}^d$ are independent,
the above distribution is equal to the one of
$$\frac{\sigma_u^2V}{\sigma_u^2V+\sigma_u^2 Z}$$
where $V = u_1^2/\sigma_u^2 \sim \Gamma(1/2)$ (and is a chi-square distribution) ) and $Z = (\sum_{i=2}^{d} u_i^2)/\sigma_u^2 \sim \Gamma((d-1)/2)$ and thus $V/(V+Z)$ follows a beta distribution 
$B(1/2,(d-1)/2)$. And the fact that $\z$ is also a unit-norm vector concludes the proof.

\end{proof}
A simulation of the random $Y$ and resulting histogram is
depicted in Figure \ref{fig:simu}.
\begin{remark} From the properties of the beta distribution, expectation and variances are given by
	$$\expectation Y=\frac{1}{d}\quad\text{and}\quad \variance Y=\frac{2(d-1)}{d^2(d+2)}$$	
\end{remark}
\begin{remark}
Note that if $\z$ is not of unit-length then $Y$ follows $\|\z\| B(1/2,(d-1)/2)$
\end{remark}

\begin{figure}
	\centering
	\includegraphics[width=5cm]{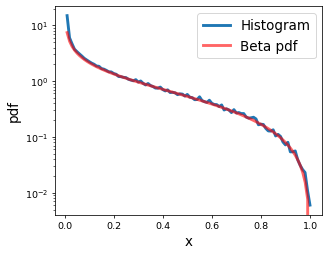}
	\caption{Estimation of the pdf of $Y$ in Lemma 1, for a fixed $\z$, based on a histogram
		over $100000$ samples of $\u$. Here, we have $d=5$.
	\label{fig:simu}}
\end{figure}

\subsection{Lemma 2 and its proof}

\begin{lemma}
		Suppose again that $\z$ is unit norm. With probability at least $1-\delta$, we have 
	\begin{align}
	\left\|\X \U - \X^\prime \U \right\|_F^2\leq w(k,\delta),
	\intertext{with}
	w(k,\delta)\doteq\frac{k}{d}+\frac{2}{3}\ln\frac{1}{\delta} + \frac{2}{d}\sqrt{k\frac{d-1}{d+2}\ln\frac{1}{\delta}}
	\end{align}
\end{lemma}
\begin{proof}
			
	First observe that:
	\begin{align*}
	H&\doteq\left\|\X \U - \X^\prime \U \right\|_F^2=\left\|(\X- \X^\prime)\U \right\|_F^2\\
	& = \left\|\z^{\top}\U\right\|_2^2=\sum_{j=1}^k\left(\z^{\top}\frac{\u_j}{\|\u_j\|}\right)^2\\
	& = \sum_{j=1}^k Y_j,\text{ where } Y_j\doteq \left(\z^{\top}\frac{\u_j}{\|\u_j\|}\right)^2.
	\end{align*}
	Therefore, $H$ is the sum of $k$ iid $B(1/2, (d-1)/2)$-distributed random variables.
	
	It is thus possible to use any inequality bounding $H$ from its mean to state a highly probable
	interval for $H$. We here use 	 inequality, that is tighter than Hoeffding inequality,
	whenever some knowledge is provided on the variance of the random variables considered. Recall
	that it states that if $Y_1,\ldots,Y_k$ and zero-mean independent RV with such that $|Y_i|\leq M$ a.s:
	$$\P\left(\sum_{j=1}^k Y_j\geq t\right)\leq \exp\left(-\frac{t^2}{2\sum_{j=1}^k\E Y_j^2+\frac{2}{3}Mt}\right)$$
	For $H$, we have $$\E H = \sum_{j=1}^k\E
	\left(\z^{\top}\frac{\u_j}{\|\u_j\|}\right)^2=\sum_{j=1}^k\frac{1}{d}=\frac{k}{d}$$
	and Bernstein's inequality gives
	$$\P\left(H\geq \frac{k}{d}+t\right)\leq \exp\left(-\frac{t^2}{2kv_d+\frac{2}{3}t}\right),$$
	where $$v_d=\frac{2(d-1)}{d^2(d+2)}$$ is the variance of each $(\z^{\top}\u_j/\|\u_j\|)^2$ beta distributed variable.
	Making the right hand side be equal to $\delta$, solving the second-order equation for $t$ give 
	that, with probability at least $1-\delta$
	$$H\leq \frac{k}{d}+\frac{2}{3}\ln\frac{1}{\delta} + \sqrt{2kv_d\ln\frac{1}{\delta}}$$
	
	The proof follows directly from Lemma~\ref{lem:single_beta} and the fact 
\end{proof}

From the above lemma, we have a probabilistic bound on the sensitivity of the random direction projection and SWD . The lower this bound is the better it is, as less noise needed
for achieving a certain $(\varepsilon,\delta)$-DP. Interestingly, the first and last terms in this bound have an inverse dependency on the \textbf{dimension}.
Hence, if the dimension of space in which the DP-SWD has to be chosen, for instance, when considering latent representation,  
a practical compromise has to be performed between a smaller bound and
a better estimation. Also remark that if $k<d$, the bound is mostly dominated by the term $\log(1/\delta)$.  Compared to other random-projection bounds \cite{tudifferentially} which have a linear  dependency in $k$. For our bound,  dimension also help in mitigating this dependency.

\subsection{Proof of the Central Limit Theorem based bound}

\begin{proof}{Proof with the Central Limit Theorem}
	According to the Central Limit Theorem --- whenever $k>30$ is the accepted rule of thumb ---
	we may consider that
	$$\frac{H}{k}\sim\normal\left(\frac{1}{d},\frac{v_d}{k}\right)$$
	i.e.
	$$\left(\frac{H}{k}-\frac{1}{d}\right)\sqrt{\frac{k}{v_d}}\sim \normal(0,1)$$
	and thus
	$$\P\left(\left(\frac{H}{k}-\frac{1}{d}\right)\sqrt{\frac{k}{v_d}}\geq t\right)\leq 1 - \Phi(t)$$
	Setting $1 - \Phi(t)=\delta$ gives $t=\Phi^{-1}(1-\delta)\doteq z_{1-\delta}$, and
	thus with probability at least $1-\delta$
	\begin{align*}
	H&\leq \frac{k}{d} + z_{1-\delta}\sqrt{k{v_d}}\\
	&=\frac{k}{d} + \frac{z_{1-\delta}}{d}\sqrt{\frac{2k(d-1)}{d+2}}
	\end{align*}
\end{proof}

\subsection{Proof of Property 2.}
\setcounter{property}{1}
\begin{property}
	$\text{DP}_\sigma\text{SWD}_q^q(\mu,\nu)$ is symmetric and satisfies the triangle inequality for $q = 1$. 
\end{property}
\begin{proof} The symmetry trivially comes from the definition of 	$\text{DP}_\sigma\text{SWD}_q^q(\mu,\nu)$ that  is 
	$$
	\text{DP}_\sigma\text{SWD}_q^q(\mu,\nu) =
	\mathbf{E}_{ \u \sim \mathbb{S}^{d-1}}  W_q^q(\mathcal{R}_\u \mu * \mathcal{N}_\sigma,\mathcal{R}_\u \nu * \mathcal{N}_\sigma ) 
	$$
	and the fact the Wasserstein distance is itself symmetric.

	Regarding the triangle inequality for $q \geq 1$, our result is based on
	a very recent result showing that the smoothed Wasserstein for $q\geq 1$ is also a metric \cite{nietert2021smooth} (Our proof is indeed valid for $q \geq 1$, as this recent result generalizes the one of \cite{goldfeld2020gaussian} ).
	Hence, we have
	\begin{eqnarray}
	\text{DP}_\sigma\text{SWD}_q(\mu,\nu) &=
	\Big[\mathbf{E}_{ \u \sim \mathbb{S}^{d-1}}  W_q^q(\mathcal{R}_\u \mu * \mathcal{N}_\sigma,\mathcal{R}_\u \nu * \mathcal{N}_\sigma )\Big]^{1/q} \nonumber
	\\&\leq	\Big[\mathbf{E}_{ \u \sim \mathbb{S}^{d-1}} \big( W_q(\mathcal{R}_\u \mu * \mathcal{N}_\sigma,\mathcal{R}_\u \xi * \mathcal{N}_\sigma ) \nonumber
	\\&+  W_q(\mathcal{R}_\u \xi * \mathcal{N}_\sigma,\mathcal{R}_\u \nu * \mathcal{N}_\sigma )\big)^q \Big]^{1/q} \nonumber
	 	\\  &\leq	\Big[\mathbf{E}_{ \u \sim \mathbb{S}^{d-1}}  W_q^q(\mathcal{R}_\u \mu * \mathcal{N}_\sigma,\mathcal{R}_\u \xi * \mathcal{N}_\sigma ) \Big]^{1/q} \nonumber
	 \\&+  \Big[\mathbf{E}_{ \u \sim \mathbb{S}^{d-1}} W_q^q(\mathcal{R}_\u \xi * \mathcal{N}_\sigma,\mathcal{R}_\u \nu * \mathcal{N}_\sigma ) \Big]^{1/q} \nonumber
	 \\ &\leq \text{DP}_\sigma\text{SWD}_q(\mu,\xi) + \text{DP}_\sigma\text{SWD}_q(\xi,\nu) \hfill~\nonumber
	\end{eqnarray}
	where the first inequality comes from the fact that the smoothed Wassertein distance  $W_q( \mu * \mathcal{N}_\sigma, \nu * \mathcal{N}_\sigma )	$ is a metric and satisfies the triangle inequality and the second one follows from
	the application of the Minkowski inequality.
	
\end{proof}

\subsection{Experimental set-up}

\subsubsection{Dataset details}
\label{sec:data}
We have considered $3$ families of domain adaptation problems based on 
Digits, VisDA, Office-31. For all these datasets,
we have  considered the natural train/test number of examples.

For the digits problem, we have used the MNIST and the USPS datasets. For MNIST-USPS and USPS-MNIST, we have respectively used 60000-7438, 7438-10000 samples. 
The VisDA 2017 problem is a $12$-class classification problem with source and target domains being simulated and real images.
The Office-31 is an object categorization problem involving $31$ classes with a total of 4652 samples. There exists $3$ domains in the problem based on
the source of the images : Amazon (A), DSLR (D) and WebCam (W). We have considered all possible pairwise source-target domains. 

For the VisDA and Office datasets, we have considered Imagenet pre-trained ResNet-50 features and our feature extractor (which is a fully-connected feedforword networks) aims at adapting those features. We have used pre-trained features
freely available at \url{https://github.com/jindongwang/transferlearning/blob/master/data/dataset.md}. 

\subsubsection{Architecture details for domain adaptations}
\label{sec:architect}
\paragraph{Digits} For the MNIST-USPS problem, the architecture of our feature extractor is composed of the two CNN layers with 32 and 20 filters of size $5\times5$. The feature extractor uses a ReLU activation function a max pooling at the first layer
and a sigmoid activation function at the second one.
For the classification head, we have used a 2-layer fully connected networks as a classifier with $100$ and $10$ units. 

\paragraph{VisDA} For the VisDA dataset, we have considered pre-trained 2048 features obtained from a ResNet-50 followed by $2$ fully connected networks with $100$ units and ReLU activations. The latent space is thus of dimension $100$. Discriminators and classifiers are also a $2$ layer fully connected networks with $100$ and respectively 1 and ``number of class'' units.

\paragraph{Office 31} For the Office dataset, we have considered pre-trained 2048 features obtained from a ResNet-50 followed by two fully connected networks with output of $100$ and $50$ units and ReLU activations. The latent space is thus of dimension $50$. Discriminators and classifiers are also a $2$ layer fully connected networks with $50$ and respectively 1 and ``number of class'' units.

For Digits, VisDA and Office 31 problems, all models have been trained using Adam with learning rate validated on the non-private  model.

\subsubsection{Architecture details for generative modelling.}

For the MNIST, FashionMNIST generative modelling problems, we have
used the implementation of MERF available at \url{https://github.com/frhrdr/dp-merf} and plugged in our 	$\text{DP}_\sigma\text{SWD}$  distance.
The generator architecture we used is the same as theirs and detailed in 
Table \ref{tab:mnist}. The optimizer is an Adam optimizer with the default $0.0001$ learning rate. The code dimension is $10$ 
and is concatenated with the one-hot encoding of the $10$ class label,
leading to an overall input distribution of $20$.

For the CelebA generative modelling, we used the implementation of \citet{nguyen2021distributional} available at \url{https://github.com/VinAIResearch/DSW}. The  generator  mixes transpose convolution and batch normalization as
described in Table \ref{tab:archiCelebA}. The optimizer is an Adam optimizer with a learning rate of $0.0005$. Again, we have just plugged
in our $\text{DP}_\sigma\text{SWD}$ distance.

\begin{table}[t]
\caption{Description of the generator for the MNIST and FashionMNIST dataset.
	\label{tab:mnist}~\\}
\begin{tabular}{ll}
	\hline
	Module & Parameters \\\hline
	FC &  20 - 200 \\
	BatchNorm & $\epsilon=10^{-5}$, momentum=0.1 \\
	FC &  200 - 784 \\
	BatchNorm & $\epsilon=10^{-5}$, momentum=0.1 \\
	Reshape & 28 x 28  \\
	upsampling & factor = 2 \\
	Convolution & 5 x 5 + ReLU \\
	Upsampling & factor = 2 \\
	Convolution & 5 x 5 + Sigmoid \\
	\hline
\end{tabular}
\end{table}

\begin{table*}	
	\centering
	\caption{Model hyperparameters and privacy for achieving
	a $\varepsilon-\delta$ privacy with $\varepsilon=10$ and $\delta$ depending on 
the size of the private dataset. The four first lines refers to the domain adaptation problems and the data to protect is the private one. The last two rows refer to the generative modelling problems. The noise $\sigma$ has been obtained using
the RDP based moment accountant of \citet{xiang2020}. \label{tab:parameter} }~\\
	\begin{tabular}{lccccccc}\hline
data & $\delta$&$d$ & $k$ & N & \#epoch& batch size & $\sigma$ \\\hline 
U-M & $10^{-5}$ & 784 & 200 & 10000 &100 & 128 & 4.74 \\
M-U & $10^{-5}$ & 784 & 200 & 7438 &100& 128 & 5.34 \\
VisDA & $10^{-5}$ & 100 & 1000 & 55387 & 50 &  128 & 6.40  \\
Office & $10^{-3}$ & 50 & 100 & 497 & 50 & 32 & 8.05 \\ \hline 
MNIST (b)& $10^{-5}$ & 784 & 1000 & 60000 & 100 & 100 & 2.94 \\
MNIST (c)& $10^{-5}$ & 784 & 1000 & 60000 & 100 & 100 & 0.84 \\
CelebA (b) & $10^{-6}$ & 8192 & 2000& 162K & 100 & 256 & 2.392 \\\hline  
CelebA (c) & $10^{-6}$ & 8192 & 2000& 162K & 100 & 256 & 0.37 \\\hline  
	\end{tabular}
\end{table*}
\begin{table}[t]
	\caption{Description of the generator for the CelebA dataset. The input code
		is of size $32$ and the output is $64 \times 64 \times 3$.
		\label{tab:archiCelebA}~\\}
	\begin{tabular}{ll}
		\hline
		Module & Parameters \\\hline
		Transpose Convolution &  32 - 512, kernel = 4x4, stride = 1 \\
		BatchNorm &  $\epsilon=10^{-5}$, momentum=0.1  \\
		ReLU & \\
		Transpose Convolution &  512 - 256, kernel = 4x4, stride = 1 \\
BatchNorm & $\epsilon=10^{-5}$, momentum=0.1  \\
ReLU & \\
		Transpose Convolution &  256 - 128, kernel = 4x4, stride = 1\\
BatchNorm &  $\epsilon=10^{-5}$, momentum=0.1 \\
ReLU & \\
		Transpose Convolution &  128 - 64, kernel = 4x4, stride = 1\\
BatchNorm &   $\epsilon=10^{-5}$, momentum=0.1\\
ReLU & \\
		Transpose Convolution &  64 - 3, kernel = 4x4, stride = 1\\
BatchNorm & $\epsilon=10^{-5}$, momentum=0.1  \\ 
Tanh & \\\hline
	\end{tabular}
\end{table}

\end{document}